\documentclass[twoside]{article}

\usepackage[accepted]{aistats2019}
\usepackage[utf8]{inputenc} 
\usepackage[T1]{fontenc}    
\usepackage{hyperref}       
\usepackage{url}            
\usepackage{booktabs}       
\usepackage{amsfonts}       
\usepackage{nicefrac}       
\usepackage{microtype}      
\usepackage{amsmath}
\usepackage{upgreek}
\usepackage{graphicx}
\usepackage{subfig}
\usepackage[colorinlistoftodos]{todonotes}
\usepackage{bm}
\usepackage{algorithm}
\usepackage{algorithmic}
\usepackage[group-separator={,}]{siunitx}
\usepackage{xfrac}

\usepackage{mathtools}
\usepackage{amsthm}
\usepackage{amssymb}
\usepackage{array,multirow}
\usepackage{breqn}

\usepackage{color}
\definecolor{darkgreen}{rgb}{0.0, 0.5, 0.0}

\newcommand{\ah}{\hat{a}}
\newcommand{\ab}{\mathbf{a}}

\newcommand{\coldef}{\vcentcolon=}

\newcommand{\pmodel}{\hat{p}_{\scriptscriptstyle \mathcal{M}}}
\newcommand{\phid}{p}

\newcommand{\ph}{\hat{p}_{t_2}}

\newcommand{\x}{\mathbf{x}}
\newcommand{\w}{\mathbf{w}}
\newcommand{\W}{\mathbf{W}}
\newcommand{\btheta}{{\uptheta}}

\newcommand{\RR}{\mathbb{R}}
\newcommand{\EE}{\mathbb{E}}

\newcommand{\Z}{\mathcal{Z}_t}

\newcommand{\Reg}{\mathcal{R}}
\newcommand{\loss}{\mathcal{L}(\w)}
\newcommand{\Loss}{\mathcal{L}(\W)}

\newcommand{\inner}[2]{#2^{\top} #1}
\newcommand{\Eq}{\mathbb{E}_{q_{t}}}

\newcommand{\act}{\ab}
\newcommand{\actp}{a_{\scriptscriptstyle +}}
\newcommand{\actn}{a_{\scriptscriptstyle -}}
\newcommand{\sur}{\xi}
\newcommand{\Id}{\mathbb{I}}
\newcommand{\lp}{G_{t_2}}

\newcommand{\thetap}{\w_{\scriptscriptstyle +}}
\newcommand{\thetan}{\w_{\scriptscriptstyle -}}

\newcommand{\LR}{{LR}}
\newcommand{\tLR}{{$t$-LR}}

\newcommand{\method}{2TLR}

\newcommand{\escort}{\hat{q}_{t_2}}
\newcommand{\zol}{\mathrm{0}\,\mathrm{\text{-}1}}

\newcolumntype{N}{>{\centering\arraybackslash}m{.82in}}
\newcolumntype{M}{>{\centering\arraybackslash}m{0.74in}}
\newcolumntype{G}{>{\centering\arraybackslash}m{0.8in}}
\usepackage{booktabs}

\usepackage{hyperref,url}
\usepackage{amsmath,amssymb,amsthm}
\usepackage{tikz}
\usepackage{xspace}
\usepackage{stmaryrd,wasysym,clrscode}
\usepackage{etex,etoolbox}
\usepackage{ifthen}
\usetikzlibrary{patterns,positioning}

\def\[#1\]{\begin{align}#1\end{align}}
\def\(#1\){\begin{align*}#1\end{align*}}

\definecolor{NAColor}{rgb}{.75,0,.75}

\newcommand{\bprf}{\begin{proof}}
\newcommand{\eprf}{\end{proof}}
\newcommand{\blem}{\begin{lemma}}
\newcommand{\elem}{\end{lemma}}

\DeclareMathOperator{\sign}{sign}

\newtheorem{theorem}{Theorem}
\newtheorem{lemma}[theorem]{Lemma}

\newtheorem{corollary}[theorem]{Corollary}

\theoremstyle{definition}

\newtheorem{definition}[theorem]{Definition}
\newtheorem{remark}[theorem]{Remark}

\begin{document}

\runningtitle{Two-temperature logistic regression based on the Tsallis divergence}
\runningauthor{Ehsan Amid, Manfred K. Warmuth,  Sriram Srinivasan}

\twocolumn[

\aistatstitle{Two-temperature logistic regression\\
based on the Tsallis divergence}
\vspace{-0.4cm}
\aistatsauthor{ Ehsan Amid$^\ast$ \And Manfred K. Warmuth$^{\ast\,,\dagger}$  \And  Sriram Srinivasan$^\ast$  }
\vspace{0.1cm}
\aistatsaddress{ \And  $\ast\,\,$University of California, Santa Cruz\\ $\dagger\,\,$Google Brain, Z\"urich\\\texttt{\{eamid, manfred, ssriniv9\}@ucsc.edu}   \And   }
]

\begin{abstract}
\vspace{-0.31cm}

We develop a variant of multiclass logistic regression that is significantly more robust to noise. The algorithm has one weight vector per class and the surrogate loss is a function of the linear
activations (one per class). The surrogate loss 
of an example with linear activation vector $\ab$ and class $c$
has the form $-\log_{t_1} \exp_{t_2} (a_c - G_{t_2}(\ab))$
where the two temperatures $t_1$ and $t_2$ ``temper'' the
    $\log$ and $\exp$, respectively, and $\lp(\ab)$ is a scalar value that
generalizes the log-partition function.
We motivate this loss using the Tsallis divergence. Our
    method allows transitioning between non-convex and
    convex losses by the choice of the temperature parameters. 
    As the temperature $t_1$ of the logarithm becomes smaller than
the temperature $t_2$ of the exponential, the surrogate loss becomes
``quasi convex''. Various tunings of the temperatures
recover previous methods and tuning the degree of
non-convexity is crucial in the experiments. In particular, quasi-convexity and boundedness of the loss provide significant robustness to the outliers. We explain this by showing that $t_1 < 1$ caps the surrogate loss and $t_2 >1$ makes the predictive
distribution have a heavy tail.

We show that the surrogate loss is Bayes-consistent, even in the non-convex case. Additionally, we provide efficient iterative algorithms for calculating the log-partition value only in a few number of iterations. Our compelling experimental results on large real-world datasets show the advantage of using the two-temperature variant in the noisy as well as the noise free case.

\end{abstract}

\begin{figure*}[t!]
\vspace{-0.5cm}
\begin{center}
\subfloat[]{\includegraphics[height=0.3\textwidth]{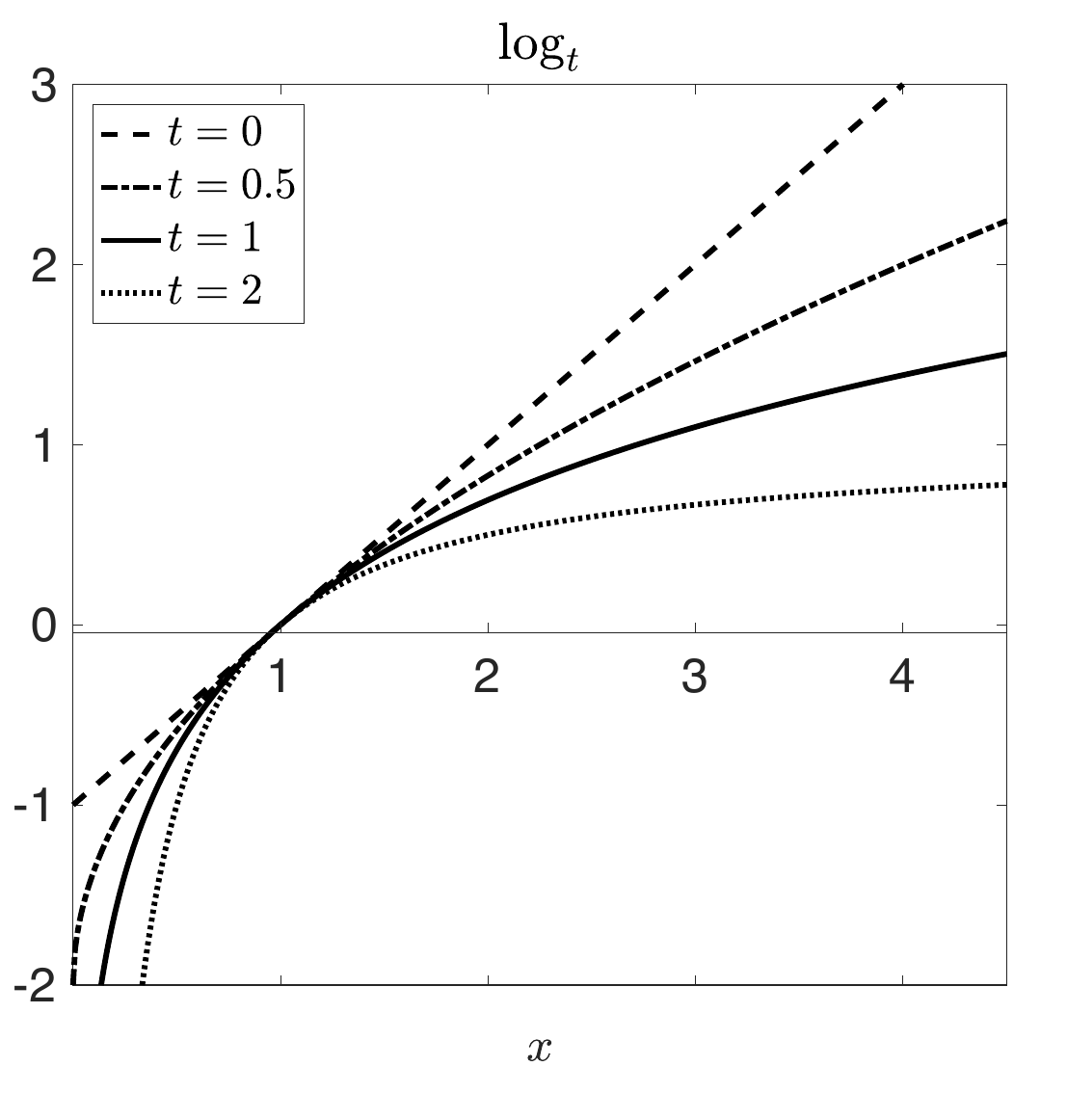}\label{fig:logt}}\,\,
\subfloat[]{\includegraphics[height=0.306\textwidth]{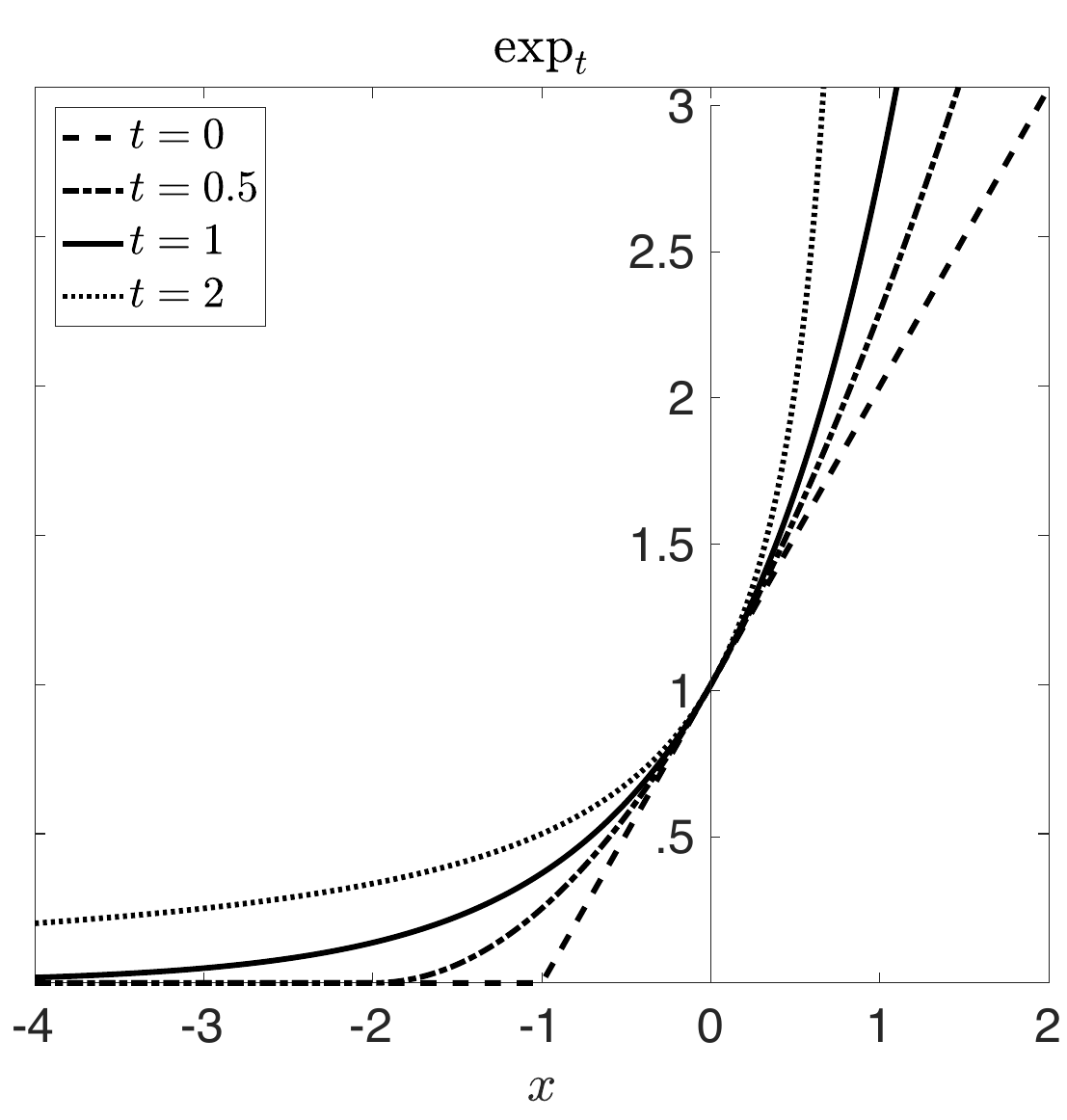}\label{fig:expt}}\,\,\,\,
\subfloat[]{\includegraphics[height=0.3\textwidth]{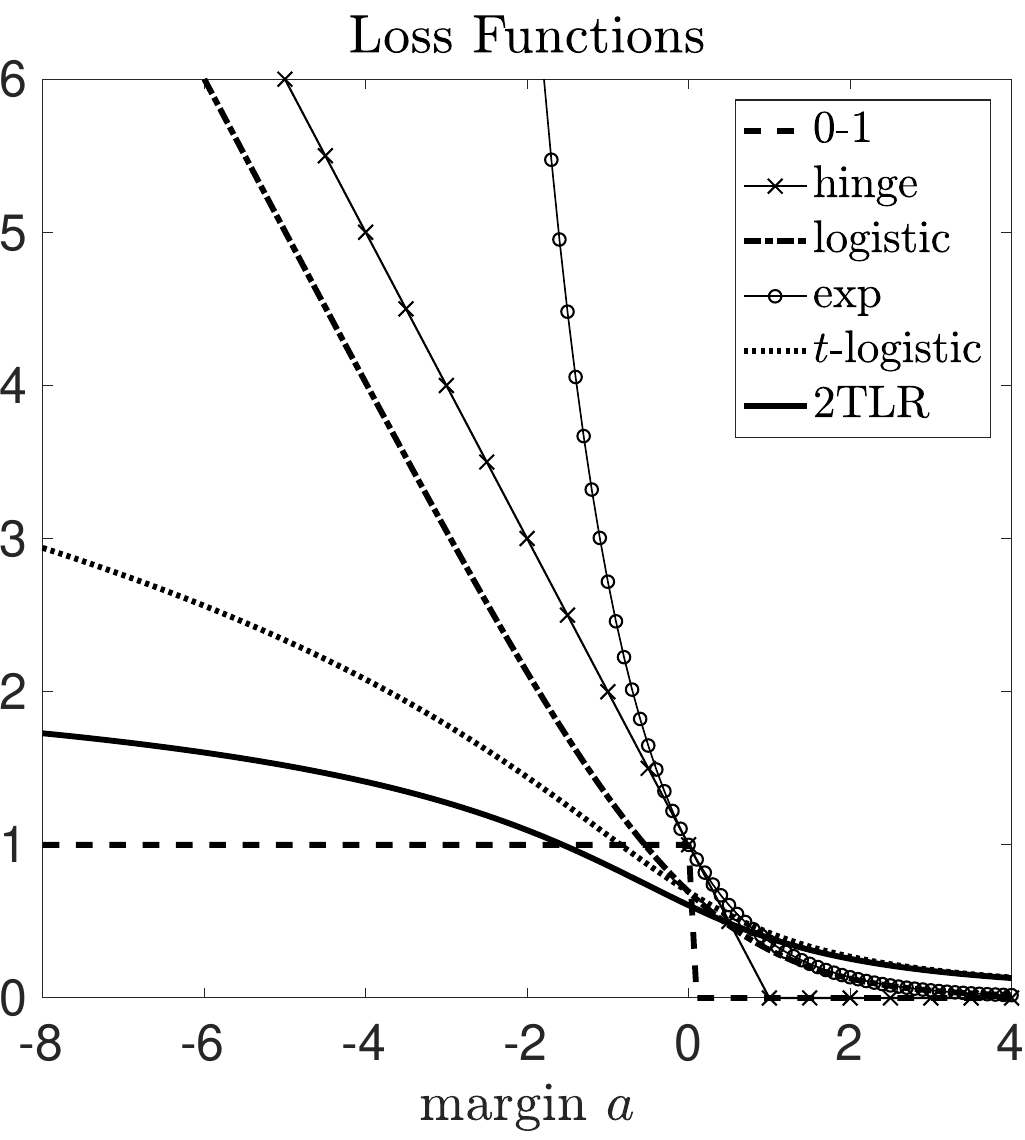}\label{fig:losses}}\hfill
     \caption{Generalized logarithm and exponential
     functions: a) $\log_t$, b) $\exp_t$, and c) Different
     loss functions for classification of a single example
     $\x$ with label $c=+1$ as a function of the margin $a
     = \w^\top \x$. The $t$-logistic loss of~\cite{tlogistic} is non-convex 
     (here $t$-logistic($a$) $=-\log \exp_t\left(a/2 - G_t(a)\right)$ with $t = 1.6$),
     but goes to $+\infty$ as margin $\rightarrow -\infty$. On the other hand, our proposed two-temperature logistic loss 
$-\log_{t_1} \exp_{t_2} (a_c - G_{t_2}(\ab))$ 
(for e.g. $t_1 = 0.6, \, t_2 = 1.6$) is upper-bounded by $1/(1-t_1) = 2.5$.}\label{fig:losses_logt_expt}
\end{center}\vspace{-0.5cm}
\end{figure*}
  
\section{Introduction}

Consider a classification problem where every instance $\x \in \RR^d$ is labeled by one class $c \in \{1,\ldots, C\}$. The goal of learning algorithm is to develop a classifier, parameterized by $\W$, which correctly predicts the class label $c$ of a given instance $\x$. 
In order to learn the optimal parameter $\W^*$ of the classifier, we minimize the \emph{regularized empirical surrogate loss} of a set of i.i.d. examples $\{(\x_n, c_n)\}_{n = 1}^N$ from the data distribution:
    $$
    \W^* = \arg \min_{\W} \; \Loss + \Reg(\W)\, ,
    $$
    where
    $$
    \Loss = \frac{1}{N} \sum_n \sur(\x_n, c_n\,\vert\,\W)\, .
    $$
Here, $\sur(\x_n, c_n\,\vert\,\W)$ denotes the \emph{surrogate
loss}, which replaces the $\zol$ loss associated with
the example $(\x_n, c_n)$. 
Also, $\W$ is a $d\times C$ weight matrix and
$\Reg(\W)$ a regularizer. 
The $c$-th column $\w_c$ is the weight vector for class $c$.
In this paper, we consider the \emph{linear activation} models 
where both the parameterized classifier 
and the surrogate loss $\sur(\x,c\,\vert\,\W)$ 
can be written as functions of the linear \emph{activation} 
vector $\act = \W^\top \x$.

Among different properties of the surrogate functions used
in practice, convexity plays an important role since it
provides the convergence guarantee of the solution to a
global minimum~\cite{urruty}. Additionally, there exist
many convex optimization packages for solving the
minimization problem efficiently~\cite{cvx, julia}. The
main drawback of the convexity is that the loss of an
individual example, e.g., for a highly misclassified
outlier point, can grow indefinitely (at least with a
linear rate) and dominate the objective function.
Therefore, it has been shown that the convex functions are
not robust to noise~\cite{long}. Specifically, Ben-David et
al.~\cite{ben2012} showed that among the convex surrogate
loss functions for linear predictors, the hinge loss has
the lowest expected misclassification error rate and any
strongly convex loss has a qualitatively worse guarantee
when compared to the hinge loss. To alleviate this problem,
several strategies have been proposed to intorduce
non-convexity into the loss function~\cite{mitchell,
robustboost, brownboost, park, croux}. More recently, Ding
et al.~\cite{tlogistic} used heavy-tailed properties of
$t$-exponential distributions to define a robust loss
function for logistic regression. The main idea behind
these techniques is to eventually ``bend down'' the loss
and give up on those points that are highly misclassified.

%

In this paper, we generalize the ideas in~\cite{tlogistic} 
for constructing a non-convex surrogate loss as the
negative log-likelihood of a $t$-exponential
distribution. Our approach is based on the Tsallis divergence which
is the natural choice of divergence for the family of
$t$-exponential distributions~\cite{amari}. 
Our definition of surrogate loss involves a generalized
logarithm and a generalized exponential function. The
generalization imbues each of these functions
with a different temperature parameter. By varying the
temperatures for the two functions, 
we transition between the convex and
more robust quasi-convex loss functions. 
More importantly, the loss function becomes bounded for certain choices of the parameters. 
Figure~\ref{fig:losses_logt_expt} illustrates the different loss functions used for classification along with an example of our proposed 
surrogate loss. 
Even though our generalization of constructing
non-convex surrogate losses is strikingly simple, our experiments
clearly show that the tail-heaviness by itself 
(as introduced in \cite{tlogistic}) is insufficient for
handling the outliers and the label noise.
More importantly, controlling the boundedness of the loss 
is an additional crucial property 
for obtaining robustness to both outliers and  label noise.
A similar bounded surrogate loss was recently developed for training deep neural networks in the presence of label noise~\cite{generalized}. 
Our contributions in this paper can be summarized as follows:
\begin{itemize}
\vspace{-0.3cm}
\setlength{\itemsep}{0.2pt}
\item We generalize the ideas in~\cite{tlogistic} and~\cite{generalized} by
introducing the two-temperature logistic regression (2TRL) 
which lets us control both the tail-heaviness as well as
boundedness of the non-convex surrogate loss.
\item We provide fast efficient iterative algorithms for
calculating the normalization constant in the $t$-exponential
probabilities.
\item We discuss the properties of the surrogate loss for 
different ranges of the two temperatures 
(the previous methods become special cases)
 and the implications of using the Tsallis divergence for parameter estimation. More specifically, we show that \emph{properness} is achieved by switching to the \emph{escort} probability of the optimizer.
\item Finally, we show that our loss is Bayes-consistent, even in the non-convex case. While many convex surrogate
losses enjoy Bayes-consistency, achieving Bayes-consistency
for non-convex losses is a highly non-trivial property and thus,
is an important consideration in designing the loss functions for classification~\cite{design}.
\end{itemize}

\section{Tsallis Entropy and Tsallis Divergence}

The $\log_t$ function with \emph{temperature} parameter $t > 0$ is defined as a generalization of the standard $\log$ function~\cite{texp1, texp2}\footnote{Note that in this section, we use $x$ as a scalar input and it should not be confused with the multivariate random variable $\x$.},
\begin{equation}
\label{eq:logt}
\log_t x = \frac{1}{1-t} (x^{1-t} - 1)\, .
\end{equation}
The $\log_t$ function is monotonically increasing and recovers the standard $\log$ function in the limit $t\rightarrow 1$. However, some properties of the $\log$ function do not generalize to $\log_t$. For instance, $\log_t ab \neq \log_t a + \log_t b$ in general. Additionally, unlike the standard $\log$ function, the $\log_t$ function 
is lower bounded by $-1/(1-t)$
for $0 < t < 1$ and upper bounded by $1/(t-1)$ for $ t > 1$
(See Figure~\ref{fig:logt}). This property has been used to design robust loss transformations for metric learning~\cite{trimap}.

Using the $\log_t$ function, we can generalize the notion of the (Shannon) entropy of a probability distribution. For a probability distribution $p(\x)$, the Tsallis entropy~\cite{tsallis} is defined as
\begin{equation}
\label{eq:tsallis-ent}
H_t(p) = \frac{\int p(\x)^t\, d\x - 1}{1-t} = \int p(\x)
\log_t  \frac{1}{p(\x)} d\x.
\end{equation}
Note that the standard entropy is recovered when $t \rightarrow 1$. Similarly, the Tsallis divergence between the distributions $p(\x)$ and $q(\x)$ can be defined as a generalization of the Kullback-Leibler (KL) divergence, that is,
\begin{equation}
\label{eq:tsallis-div}
D_t(p\Vert q) = -\int p(\x) \log_t \frac{q(\x)}{p(\x)}\,
d\x.
\end{equation}
Note that the KL divergence is also recovered in the limit $t \rightarrow 1$. 
We also define the $\exp_t$ function as the inverse of $\log_t$
(See Figure~\ref{fig:expt}):
\begin{equation} \label{eq:expt}
\exp_t(x) = [1 + (1-t)\,x]_+^{1/(1-t)}\,\, ,
\end{equation}
where $[\,\cdot\,]_+ = \max(\,\cdot\,, 0)$. Again the vanilla $\exp$ function is the $t\rightarrow 1$ limit. 
An important property of the $\exp_t$ function is its heavier tail compared to $\exp$ for values of $ t > 1$ (see Figure~\ref{fig:expt}). This property leads to definition of a class of generalized distributions under the $\exp_t$ function, called the $t$-exponential family of distributions with vector of sufficient statistics $\x$,
\begin{equation}
\label{eq:t-exp-family}
p_t(\x\,\vert\,\btheta) = \exp_t(\inner{\x}{\btheta} -
G_t(\btheta)), \quad  \text{for }t > 0\, .
\end{equation}
Here $\btheta$ is called the \emph {canonical parameter} and the convex function $G_t(\btheta)$, called the \emph{log-partition function},  ensures that the distribution is normalized, that is, 
\begin{equation}
\label{eq:log-part}
\int \exp_t(\inner{\x}{\btheta} - G_t(\btheta))\, d\x= 1\, .
\end{equation}
An important distribution related to the $t$-exponential distribution~\eqref{eq:t-exp-family} is called the \emph{escort distribution} and is defined as
\begin{align}
& q_t(\x\,\vert\, \btheta) = \frac{1}{\Z(\btheta)}\,
\exp_t(\inner{\x}{\btheta} - G_t(\btheta))^t,
\end{align}
where
$$
\Z(\btheta) = \int \exp_t(\inner{\x}{\btheta} - G_t(\btheta))^t\, d\x\,.
$$
Here $\Z(\btheta)$ is the normalization factor. 
It is easy to see that~\cite{amari}
\begin{equation}
\label{eq:grad-G}
    \!\!\!
\nabla G_t(\btheta) = \Eq[\x] = \frac{1}{\Z(\btheta)}\! \int
\!\!\x \exp_t(\x^\top\!\btheta\!-\!G_t(\btheta))^t \,d\x.
\end{equation}
As~\eqref{eq:grad-G} suggest, escort probabilities appear when calculating the gradient of the loss, as we will see in the later sections. When dealing with $t$-exponential distributions, the Tsallis entropy and divergence take the role of Shannon entropy and KL divergence respectively, 
for the vanilla exponential family (See e.g.~\cite{amari}).

\section{Two-temperature Logistic Regression}

Let $\act = \inner{\x}{\W}$. Following the discussion on the heavy-tail properties of the $t$-exponential family of distributions in~\cite{tlogistic}, we model the conditional probability of the class $c$ given input $\x$ with a $t$-exponential distribution with temperature $t_2$:
\begin{align}
\ph(c\,|\,\x, \W) & = \exp_{t_2}(\inner{\x}{\w_c}- \lp(\inner{\x}{\W}))
 \nonumber\\
    & = \exp_{t_2}\left(a_c- \lp(\act)\right)\label{eq:probs-multi}\, ,
\end{align}
where the log-partition function $\lp(\act)$ 
ensures that the probabilities sum up to $1$, that is,
\begin{equation}
\label{eq:log-part-eq}
\sum_c \exp_{t_2}(a_c- \lp(\act)) = 1\, .
\end{equation}
This definition for the conditional probabilities is similar to the ones given in~\cite{tlogistic}. The definition~\eqref{eq:probs-multi} also includes the softmax probabilities as a special case when $t_2 = 1$:
\begin{align}
    \nonumber
    \hat{p}_1 (c\,\vert\,\x, \W) &=  
\exp(a_c- \overbrace{\log\sum_j \exp(a_j)}^{G_1(\act)}) 
    \\&= \frac{\exp(a_c)}{\sum_j \exp(a_j)}\, .
\label{eq:probs-softmae}
\end{align}
In order to adopt the heavy-tail properties of
$t$-exponential distribution, we are mainly interested in
the values of $ t_2 > 1$. However, for values of $t_2 \neq
1$, the log-partition function $\lp(\act)$ does not have a
closed form solution in general and must be calculated
numerically: We provide an iterative method for
computing $G_{t_2}(\act)$ efficiently (Algorithm~\ref{alg:g-multi}).

\begin{algorithm}[t]
\caption{Iterative algorithm for computing $G_t$ for multiclass \method.}
\label{alg:g-multi}
\begin{algorithmic}
\STATE {\bfseries Input:} Vector of activations $\act$, temperature $t > 1$
\STATE {\bfseries Output:} $G_t(\act)$
\STATE $\mu \gets \max(\act)$
\STATE $\tilde{\ab} \gets \act-\mu$
\WHILE{$\tilde{\ab}$ not converged}
\STATE $Z(\tilde{\ab})\gets \sum_{c=1}^C \exp_{t}(\tilde{a}_c)$
\STATE$\tilde{\ab}\gets Z(\tilde{\ab})^{1-t} (\act-\mu)$
\ENDWHILE
\STATE $G_t(\ab)\gets -\log_t(1/Z(\tilde{\ab}))+\mu$
\end{algorithmic}
\end{algorithm}

Given the prediction probabilities~\eqref{eq:probs-multi} 
in the form of a $t_2$-exponential distribution, 
we can now define the loss between the empirical label
distribution $p_e(c\,\vert\, \x_n) = \Id_{c=c_n}$, and the
prediction $\ph(c\,\vert\,\x_n)$ using a 
sum of Tsallis divergences with temperature $t_1$:
\begin{align}
\label{eq:sum-tsallis-multi}
\Loss &= -\frac{1}{N}\!\!\sum_n \!\sum_c
p_e(c\,\vert\,\x_n)\log_{t_1}\!\! \frac{\ph(c\,\vert\,\x_n, \W)}{p_e(c\,\vert\,\x_n)}\nonumber\\
& = -\frac{1}{N} \sum_n \!\sum_c \Id_{c = c_n}\!\log_{t_1}\!\!
    \frac{\ph(c\,\vert\,\x_n, \W)}{\Id_{c = c_n}}.
\end{align}
Justified by a limit argument, $0\times \log_t 0 = 0\times \log_t \infty = 0$, the loss \eqref{eq:sum-tsallis-multi} simplifies to 
\begin{align}
    \nonumber
    &\Loss  = -\frac{1}{N}\sum_n \log_{t_1} \ph(c_n\,\vert\,\x_n, \W)\\
    &= \frac{1}{N}\sum_n 
\underbrace{[-\log_{t_1} \exp_{t_2}(\inner{\x_n}{\w_{c_n}}\!-\!
\lp(\inner{\x_n}{\W}))]}_{\sur_{t_1}^{t_2}(\x_n,c_n\,\vert\, \W)}.
\label{eq:main-loss-multi}
\end{align}
We refer to the classification algorithm with the loss
defined in~\eqref{eq:main-loss-multi} as
\emph{Two-Temperature Logistic Regression} (\method). The gradient of the loss with respect to the $c$-th parameter $\w_c$ can be written as
\begin{align}
    \nonumber
        &\!\! \nabla_{\w_c} \Loss = 
    \\  &\!\!  -\!\!\sum_n \ph(c_n\,\vert\,\x_n,\!\!\W)^{t_2 - t_1}
    \!\Big(\Id_{c = c_n}\!\! -\!
  \escort(c\,\vert\,\x_n,\!\!\W) \Big) \x_n ,
\label{eq:grad-loss-multi} 
\end{align}
where $$\escort(c\,\vert\,\x,\!\!\W) = \frac{\exp_{t_2}(a_c- \lp(\act))^{t_2}}{\sum_j
\exp_{t_2}(a_j- \lp(\act))^{t_2}} \sim \ph(c\,\vert\,\x,\!\!\W)^{t_2}$$
is the escort distribution of $\ph(c\,\vert\,\x,\!\W)$. 

We are mainly interested in $0 < t_1 < 1$ because for this
range, the loss of each individual observation becomes
capped by the constant $1/(1-t_1)$. 
As we show in the experiments, 
the boundedness of loss provides significant improvement in handling noisy observations. 
Note that the gradient of the loss of the $n$-th observation contains an \emph{importance factor} of the form $\ph(c_n\,\vert\,\x_n, \W)^{t_2 - t_1}$
that depends on the conditional probability of the $n$-th observation and the temperature gap $t_2 - t_1$. 
Note that for $t_2 > t_1$, 
the temperature gap is non-negative and the importance
factors dampen the gradient of those observations 
that have small probabilities towards zero. 
Also the loss of each observation is bounded only for values of $0 < t_1 < 1$. 
On the other hand, the importance factors vanish when $t_1 = t_2$. 
In particular, it vanishes for standard logistic regression
(i.e. when $t_1 = t_2 = 1$).

Next we focus on the binary classification 
and analyze the properties of the surrogate loss in this case.

\section{Binary Classification}

For $C = 2$, we use the classes $c \in \{\pm 1\}$
and denote the parameter vector as
$\W = [\thetap, \thetan]$ and linear activations as $\act = [\inner{\x}{\thetap},\, \inner{\x}{\thetan}]^\top = [\actp, \actn]^\top$.
Similar to~\eqref{eq:probs-multi}, we can define the probabilities as
\begin{align}
    \nonumber
\ph(c = \pm1|\x) & = \exp_{t_2}(\inner{\x}{\w_{\pm}}- \lp(\inner{\x}{\W}))\\
& =  \exp_{t_2}(a_{\pm}- \lp(\act))\,.
\label{eq:probs-pos}
\end{align}
The log-partition function $\lp(\act)$ ensures that the two probabilities 
sum to $1$. 
It is easy to see that for any constant $b$, 
$\lp(\act + b\, \mathbf{1}) = \lp(\act) + b\, \mathbf{1}$. 
Therefore we can simplify the margin vector $\act$ by subtracting the mean of the inner-products $\frac{\inner{\x}{\thetap} + \inner{\x}{\thetan}}{2}$, that is,
 $\act = [\frac{\inner{\x}{(\thetap - \thetan)}}{2},\, -\frac{\inner{\x}{(\thetap - \thetan)}}{2}]^\top = [\frac{\inner{\x}{\w}}{2},\, -\frac{\inner{\x}{\w}}{2}]^\top = [\frac{a}{2}, -\frac{a}{2}]^\top$, where we define $\w = \thetap - \thetan$. Thus, we can write the probabilities in the following compact form
\begin{align*}
\ph(c\,\vert\,\x, \w) & = \exp_{t_2}
                            (\frac{c}{2}\,\overbrace{\inner{\x}{\w}}^{a}- \lp(\inner{\x}{\w}))\, 
                        . 
\end{align*}
This definition contains the logistic
probabilities as the special case when $t_2 = 1$:
\begin{align*}
\hat{p}_1 (c\,\vert\,\x) & = \frac{\exp(\frac{c}{2}\,
a)}{\exp(\frac{c}{2}\, a) + \exp_{t_2}(\frac{-c}{2}\, a)} =
\frac{1}{1 + \exp(-c\, a)}\, ,
\end{align*}
since $G_1(a) = \log \left(\exp \frac{a}{2} + \exp
\frac{-a}{2} \right)$. For $t_2 \neq 1$, $\lp(a)$ does not
have a closed form solution\footnote{Except for $t_2=2$.} and we provide a
variant of the iterative algorithm for calculating $G_t(a)$ 
for the binary case (Algorithm~\ref{alg:g-bin}).

Following similar steps as in~\eqref{eq:sum-tsallis-multi}, we can write the loss for the binary case as
\begin{align}
\label{eq:main-loss}
\loss & =  \sum_n 
\underbrace{-\log_{t_1} \exp_{t_2}(\frac{c_n}{2}\,a_n-
\lp(a_n))}_{\sur_{t_1}^{t_2}(\x_n,y_n\,\vert\, \w)} \, .
\end{align}
where $a_n = \inner{\x_n}{\w}$. For $t_1 = t_2 = 1$, the above loss is the
standard logistic regression loss. Also for $t_1=1$ and
$t_2=t>1$, the above becomes the $t$-logistic loss of~\cite{tlogistic}. 
The gradient of the loss~\eqref{eq:main-loss} wrt $\w$ is
\begin{align*}
    &\nabla \loss  = 
    \\&-\frac{1}{2}\sum_n \ph(c_n\,\vert\,\x_n, \w)^{t_2 - t_1}
\Big(c_n - \sum_c c\, \hat{q}_{t_2} (c\,|\x, \w)\Big)\, \x_n\, ,
\end{align*}
where $\hat{q}_{t_2} (c\,\vert\,\x, \w) \sim \ph(c\,\vert\,\x, \w)^{t_2}$ is the escort distribution.

\subsection{Properties}

The curvature of the two-temperature loss function $
\sur_{t_1}^{t_2}(\x,y\,\vert\, \btheta)$ depends on the
choice of the temperature parameters $t_1$ and $t_2$. For
certain choices, we still have convex losses while for the
others, the loss function shows a quasi-convex behavior.
The properties of the loss function are summarized below. 
Without loss of generality, we assume $c = +1$.
\begin{remark}
\label{th:props}
The loss function $\sur_{t_1}^{t_2}(\x, c\,\vert\,\w) = - \log_{t_1}  \exp_{t_2}(\frac{a}{2} - \lp(a))$ has the following properties:
\begin{enumerate}
\item For values of $t_1 \geq t_2$ and $t_1 \geq 1$, the
loss function is convex. Specifically, for $t_1 = t_2 = t \geq
1$, we have the convex loss
\begin{equation}
\sur_{t}^{t}(\x,c\,\vert\, \btheta) = G_t(a) - \frac{a}{2}\, .
\end{equation}
Moreover, the curvature of the function increases with the
temperature gap $t_1 - t_2 > 0$.
\item The function is quasi-convex for $t_1 < t_2$ or for
    any $t_2 \ge 0$ when $t_1 < 1$.
\end{enumerate}
\end{remark}
The proof is provided in the Appendix~\ref{app:props}.

\begin{algorithm}[t]
\caption{Iterative algorithm for computing $G_t$ for binary
\method.}
\label{alg:g-bin}
\begin{algorithmic}
\STATE {\bfseries Input:} Activation $a > 0$, temperature $t > 1$
\STATE {\bfseries Output:} $G_t(a)$
    \IF{$t==2$}\STATE $G_t(a) \gets \sqrt{\sfrac{a^2}{4} + 1}$\\ \RETURN \ENDIF
\STATE $\tilde{a} \gets a$
\WHILE{$\tilde{a}$ not converged}
\STATE $Z(\tilde{a}) \gets 1+\exp_{t}(-\tilde{a})$\;
\STATE $\tilde{a} \gets Z(\tilde{a})^{1-t} a$\;
\ENDWHILE
\STATE $G_t(a)\gets -\log_t(1/Z(\tilde{a}))+ \sfrac{a}{2}$
\end{algorithmic}
\end{algorithm}

\vspace{0.3cm}
\section{Implications of Using the Tsallis Divergence}
\label{sec:tsallis}

We briefly discuss the implicit assumptions behind using the Tsallis divergence for parameter estimation. Consider modeling the (unknown) posterior distribution $\phid(c\,\vert\x)$ for the set of random variables $(\x, c) \in \RR^d \times \{1,\ldots,C\}$ using a discriminative model $\pmodel(c\,\vert\x)$. For this purpose, consider minimizing the expected Tsallis divergence between the class posterior distribution of the data and the predicted posterior probabilities, that is,
\begin{subequations}
\label{eq:tsallis-implication}
\begin{align}
&\EE_{\x}\left[-\sum_{c} \phid(c\,\vert \x)\log_t\frac{\pmodel(c\,\vert \x)}{\phid(c\,\vert\x)}\right]\label{subeq:tsallis-expect2}\\
    & = \EE_{\x}\bigg[\sum_{c} \phid(c\,\vert \x)^t \Big[\log_t \phid(c\,\vert\x) - \log_t \pmodel(c\,\vert \x)\Big]\bigg]\label{subeq:tsallis-int}\\
& = -H_t - \int \sum_{c} \phid(c\,\vert \x)^t\log_t \pmodel(c\,\vert \x)\, \phid(\x)\, d\x\label{subeq:tsallis-int-div}\\
& \approx -H_t - \sum_n \sum_{c} \Id_{c = c_n}\, \log_t \pmodel(c_n\,\vert \x_n)\label{subeq:tsallis-mc}\\ 
&= -H_t - \sum_n \log_t \pmodel(c_n\,\vert \x_n)\label{subeq:tsallis-final}\, ,
\end{align}
\end{subequations}
in which $H_t = -\int \sum_{c} \phid(c\,\vert \x)^t \log_t \phid(c\,\vert\x)\, \phid(\x)\, d\x = \EE_{\x}\left[\sum_{c} \phid(c\,\vert \x) \log_t \frac{1}{\phid(c\,\vert \x)}\right]$ is the expected Tsallis entropy of the posterior distribution $\phid(c\,\vert \x)$ and is a constant. Note that from~\eqref{subeq:tsallis-expect2} to~\eqref{subeq:tsallis-int} we use the property $\log_t(u/v) = u^t (\log_t u - \log_t v)$ of the $\log_t$ function and  from~\eqref{subeq:tsallis-int-div} to~\eqref{subeq:tsallis-mc} we perform a Monte Carlo approximation of the integral and sum using a set of samples $\{\x_n, c_n\}$. Therefore, we can eliminate the second sum in~\eqref{subeq:tsallis-mc} and only keep the terms corresponding to the observed labels, as in~\eqref{subeq:tsallis-final}. However, indeed, minimizing the sum in~\eqref{subeq:tsallis-final} involves the implicit assumption that the $c$ samples are drawn from the tempered conditional distribution $\sim \phid(c\,\vert \x)^t$ and therefore, the minimizer for a single example $\x$ solves $\pmodel^*(c\,\vert \x) \sim \phid(c\,\vert \x)^{1/t}$. Thus, as a consequence of
using the Tsallis divergence
in~\eqref{subeq:tsallis-final}, the surrogate loss
$\sur_{t_1}^{t_2}(\x_n,c_n\,\vert\, \w)$ is not
\emph{proper}~\cite{proper}, i.e., $\ph(c\,|\,\x, \W^*) \neq \phid(c\,\vert \x)$. However, simply enough, the \emph{escort} probabilities $\sim \ph(c\,|\,\x, \W^*)^{t_1}$ match to the correct conditional probabilities. In the case of $t=1$, the Tsallis divergence reduces to the KL-divergence and we recover the maximum-likelihood estimation $- \sum_n \log \pmodel(c_n\vert \x_n) = -\log \prod_n \pmodel(c_n\vert \x_n)$ and $\pmodel^*(c\vert \x) = \phid(c\vert \x)$.

Although the properness of the loss function may be important in density estimation applications, for the classification problem, the estimated posterior probabilities are irrelevant as long as the class label is predicted correctly. Thus, we are mainly interested in the Bayes-consistency property of the loss~\cite{bayes, bayes-multi}, which guarantees that at the solution, the correct label can be predicted using the $\arg \max$ of the margin vector $\act$.

\begin{figure*}[t!]
\vspace{-0.5cm}
\begin{center}
	\subfloat{\includegraphics[width=0.24\textwidth]{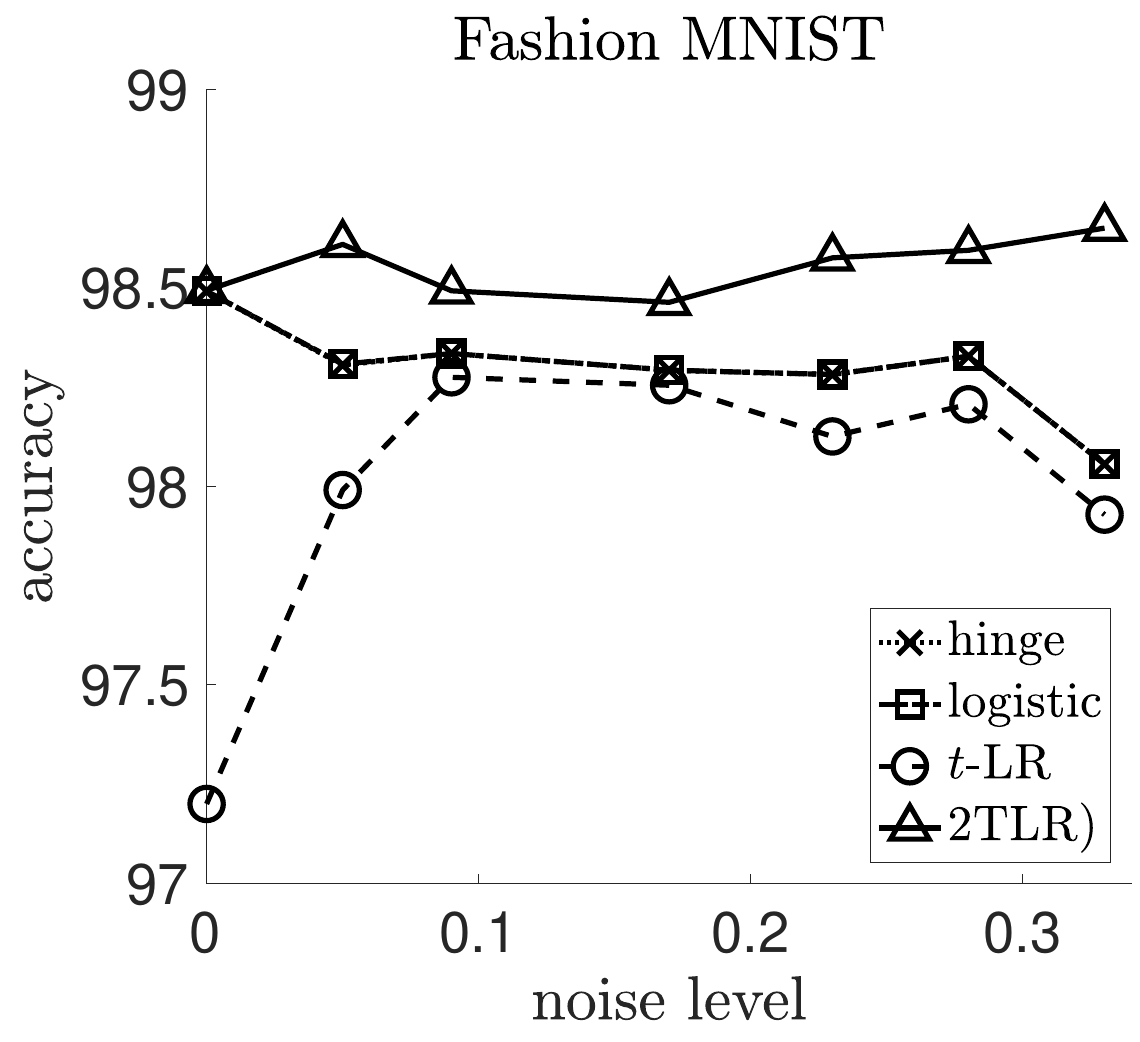}}\,
	\subfloat{\includegraphics[width=0.24\textwidth]{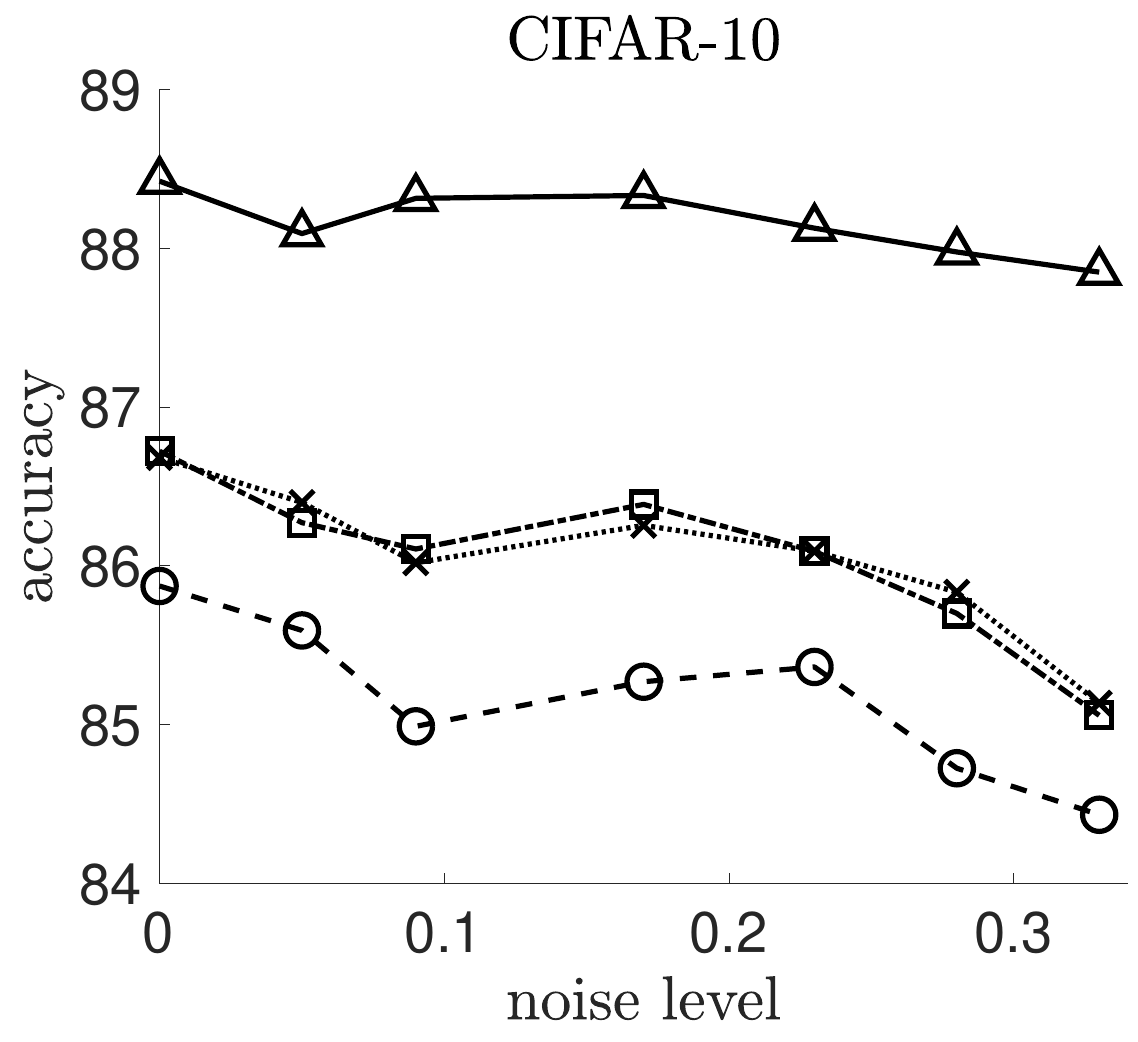}}\,
	\subfloat{\includegraphics[width=0.24\textwidth]{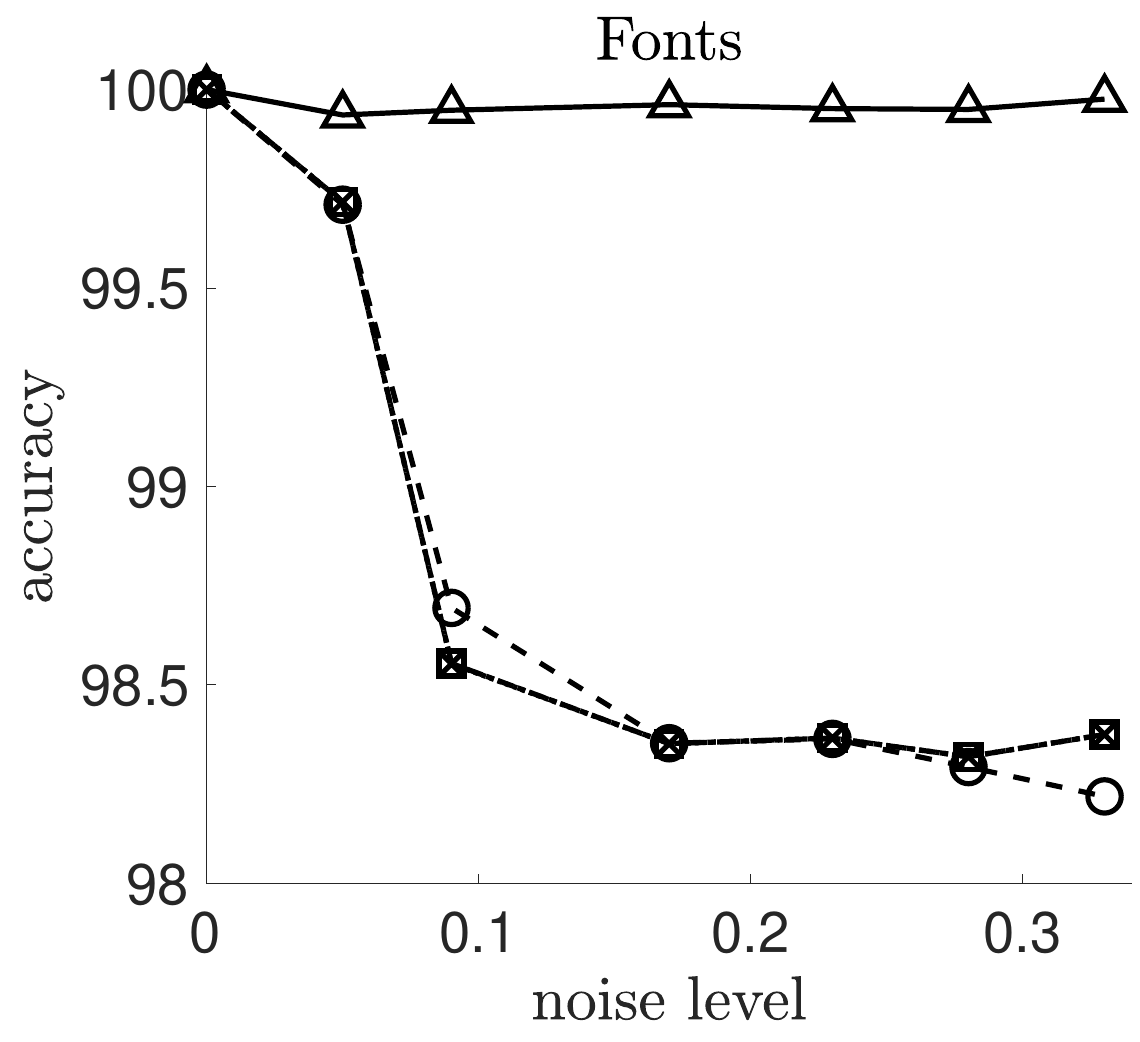}}\,
	\subfloat{\includegraphics[width=0.24\textwidth]{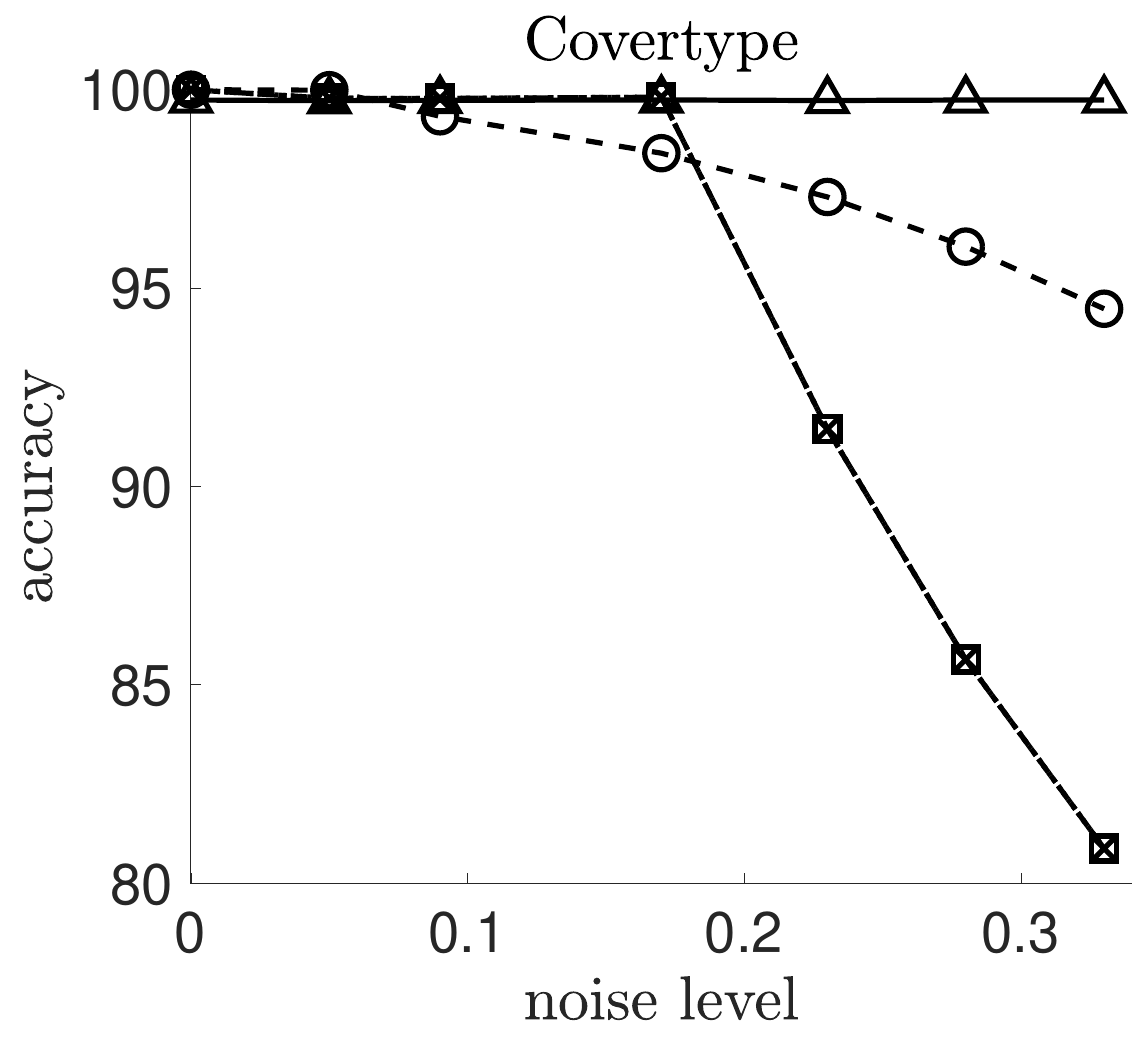}}\hfill
     \caption{The classification accuracy in the presence of instance noise. The errorbars are small and not shown to avoid clutter.}\label{fig:outlier}
\end{center}
\end{figure*}

\begin{table*}[!t]
\noindent
\centering
\begin{tabular}{c G N N N N}
\toprule
\multicolumn{1}{ N }{\begin{tabular}{c} \textbf{Dataset}\\(\#instances, \#dim)\end{tabular}} & \textbf{Noise
Type}  & \multicolumn{4}{c }{\textbf{Classification Accuracy (\%)}} \\  
\cmidrule(lr){3-6}
 \multicolumn{2}{N }{\textbf{}}  & \textbf{hinge} & \textbf{logistic} &
\textbf{$\bm{t}$-LR} & \textbf{\method}\\
\midrule
   \multirow{3}{*}{\begin{tabular}{c} \textbf{Fashion MNIST}\\($20$K, $784$)\end{tabular}} & random & $96.42\pm0.59$ & $96.42\pm0.59$ & $94.09\pm0.48$ & $\mathbf{99.80\pm0.12}$ \\
& small-margin & $98.50\pm0.26$ & $98.50\pm0.26$ & $97.35\pm0.42$ & $\mathbf{99.13\pm0.37}$ \\
    & large-margin & $96.42\pm0.59$ & $96.42\pm0.59$ & $94.09\pm0.48$ & $\mathbf{99.80\pm0.12}$ \\
 \midrule
 \multirow{3}{*}{\begin{tabular}{c} \textbf{CIFAR-10}\\($10.8$K, $\num{1024}$)\end{tabular}} & random & $84.27\pm1.12$ & $84.39\pm1.17$ & $82.11\pm1.01$ & $\mathbf{87.75\pm1.40}$ \\
& small-margin & $84.94\pm0.97$ & $84.94\pm0.99$ & $84.22\pm0.79$ & $\mathbf{86.28\pm1.18}$ \\
 & large-margin & $77.79\pm1.20$ & $77.77\pm1.20$ & $72.58\pm1.44$ & $\mathbf{88.56\pm1.20}$ \\
 \midrule
  \multirow{3}{*}{\begin{tabular}{c} \textbf{Fonts}\\($143$K, $411$)\end{tabular}} & random & $83.78\pm0.28$ & $83.78\pm0.28$ & $\mathbf{84.14\pm0.27}$ & $\mathbf{84.14\pm0.27}$ \\
& small-margin & $\mathbf{83.60\pm0.34}$ & $\mathbf{83.60\pm0.34}$ & $83.38\pm0.36$ & $\mathbf{83.60\pm0.34}$ \\
 & large-margin & $72.39\pm0.32$ & $72.39\pm0.32$ & $\mathbf{72.61\pm0.30}$ & $\mathbf{72.61\pm0.30}$ \\
 \midrule
  \multirow{3}{*}{\begin{tabular}{c} \textbf{Covertype}\\($287$K, $54$)\end{tabular}} & random & $97.52\pm0.88$ & $97.52\pm0.88$ & $\mathbf{99.26\pm0.05}$ & $\mathbf{99.26\pm0.05}$ \\
& small-margin & $96.79\pm0.11$ & $96.79\pm0.11$ & $\mathbf{97.25\pm0.05}$ & $\mathbf{97.25\pm0.05}$ \\
 & large-margin  & $83.59\pm0.24$ & $83.59\pm0.24$ & $84.79\pm0.20$ & $\mathbf{94.03\pm0.13}$ \\
\bottomrule
\end{tabular}
    \caption{Classification accuracy with $10\%$ label noise. The noise is added by selecting the points in three different manners: 1) Random: points are selected uniformly at random, 2) Small-Margin (SM): the points having smallest margin are selected, 3) Large-Margin (LM): the points having largest margin are selected.}\label{tab:label-noise}
\end{table*}

\section{Bayes-consistency}
\label{sec:bayes}


We use the results from Zhang et al.~\cite{coherence} to show the Bayes-consistency of the multiclass class  case. 

\begin{definition}[Zhang et al.~\cite{coherence}]
\label{lem:bayes-consist}
A surrogate loss $\sur(\act, c)$ w.r.t. a margin $\act = [a_1, \ldots, a_m]^\top$ with the additional constraint $\sum_c a_c = 0$ is said to be Bayes-consistent if for all possible label probability distributions $p(c\, \vert\, \x)$ the following conditions are satisfied:
\begin{enumerate}
\item The minimization problem $\act^* = \arg \min_{\act} \sum_c p(c \vert \x)\, \sur(\act, c)$
has a unique solution for all $\x \in \RR^d$, and
\item $\arg \max_c a^*_c = \arg \max_c p(c\, \vert\, \x)$ for all $\x \in \RR^d$.
\end{enumerate}
\end{definition}

We now prove the following. 

\begin{theorem}
\label{th:multi-bayes}
The multiclass surrogate loss $\sur_{t_1}^{t_2}(\x, c\vert\,\W) = -\log_{t_1} \exp_{t_2} (a_c - \lp(\act))$ is Bayes-consistent.
\end{theorem}
{\it Proof.}$\;$
The minimizer of the expectation
\begin{equation}
\label{eq:expectation-surrogate}
-\sum_c p(c\, \vert\, \x)\, \log_{t_1} \exp_{t_2} (a_c - \lp(\act))\,
\end{equation}
has the unique solution $\act^*$ such that $\exp_{t_2} (a^*_c - \lp(\act^*)) \propto p(c\, \vert\, \x)^{1/t_1}$. Note that the minimizer is unique because  $\exp_{t_2}$ is an injective function and therefore any other minimizer $\act^{**}$ must satisfy the following: $a^*_c - \lp(\act^*)= a^{**}_c - \lp(\act^{**})$ for all $c \in \{1,\ldots, C\}$. Enforcing the constraint\footnote{Note that we can always enforce the constraint $\sum_c a_c = 0$ by adding and subtracting the constant vector of mean value $\left(\frac{1}{C} \sum_c a_c\right)\mathbf{1}$ without changing the probabilities since $\lp(\act + b\, \mathbf{1}) = \lp(\act) + b\, \mathbf{1}$ for any constant $b$.} $\sum_c a^*_c =  \sum_c a^{**}_c = 0$ yields $\act^* = \act^{**}$. 
Finally, monotonicity of $\exp_{t_2}$ function implies
\begin{multline}
\arg \max_c a^*_c \!=\! \arg\max_c\, \exp_{t_2} (a^*_c -
\lp(\act^*))\nonumber\\
\!=\! \arg \max_c p(c\, \vert \x)^{1/t_1}\! =\!
\arg \max_c p(c\, \vert \x)\, . \;\square
\end{multline} 

The result of Theorem~\ref{th:multi-bayes}, 
i.e. $\ph(\x, c\vert \W^*) \propto p(c \vert \x)^{1/t_1}$, 
is the direct consequence of using the sum of Tsallis divergences between the observed class distributions and the predicted class probabilities, as discussed in the previous section. However, the $\arg\max$ operator is invariant with respect to the positive powers and thus, we still achieve Bayes-consistency.


\begin{corollary} 
\label{th:bin-bayes}
The binary surrogate loss $\sur_{t_1}^{t_2}(\x, c\vert\, \w)$ is Bayes-consistent.
\end{corollary}

Note that because of the form of the margin vector $\act = [a, -a]^\top$ in the binary case, the $\arg\max$ operator is equivalent to $\sign(a)$. Therefore, the given new points can simply be classified using the sign of the activation, without explicitly calculating the probabilities.

\begin{table*}[!t]
\vspace{-0.2cm}
\noindent
\centering
\begin{tabular}{c M M M M}
\toprule
\multicolumn{1}{ N }{\begin{tabular}{c} \textbf{Dataset}\\(\#instances, \#dim)\end{tabular}} & \multicolumn{4}{c }{\textbf{Runtime (s)}} \\  
\cmidrule(lr){2-5}
 \multicolumn{1}{N }{\textbf{}}  & \textbf{hinge} & \textbf{logistic} &
\textbf{$\bm{t}$-LR} & \textbf{\method}\\
\midrule
{\begin{tabular}{c} \textbf{Fashion MNIST}\\($20$K, $784$) \end{tabular}} & $4.40\pm0.28$ & $4.57\pm0.12$ & $7.02\pm0.29$ & $7.35\pm1.21$ \\
\midrule
{\begin{tabular}{c} \textbf{CIFAR-10}\\($10.8$K, $\num{1024}$)\end{tabular}} & $31.90\pm0.22$ & $31.86\pm0.29$ & $35.08\pm0.81$ & $28.34\pm10.56$ \\
\midrule
{\begin{tabular}{c} \textbf{Fonts}\\($143$K, $411$)\end{tabular}} & $49.47\pm4.92$ & $49.75\pm4.99$ & $82.85\pm4.98$ & $58.78\pm6.14$ \\
\midrule
{\begin{tabular}{c} \textbf{Covertype}\\($287$K, $54$)\end{tabular}} & $6.45\pm0.17$ & $6.42\pm0.18$ & $66.66\pm1.20$ & $24.53\pm1.24$ \\
\bottomrule
\end{tabular}
\caption{Runtime of the different algorithms in seconds.}\label{tab:runtime}
\vspace{-0.3cm}
    \end{table*}

\section{Experiments}
\label{sec:exp}

We compare the binary classification accuracy when minimizing
the following losses:
our two-temperature surrogate loss (\method), 
vanilla logistic regression (\LR), 
hinge loss, 
and $t$-logistic regression (\tLR). 
We do not compare our results to the method recently proposed by Feng et al.~\cite{rolr} which is based on detecting and removing the outliers in the dataset. The method in~\cite{rolr} makes strict assumptions about the type of the generative distribution,
the availability of the noise variance and requires an
upper-bound on the number of outliers. These assumptions
make their method impractical for real-world applications.

Our experiments are for the following data sets: 
1) Fashion MNIST~\footnote{Available at: \url{https://github.com/zalandoresearch/fashion-mnist}}, 2) CIFAR-10~\footnote{Available at: \url{https://www.cs.toronto.edu/~kriz/cifar.htmlt}}, 3) Character Font Images~\footnote{\label{fn:uci} From the UCI repository.}, and 4) Covertype\textsuperscript{\ref{fn:uci}}.
For each dataset, we randomly pick two classes such that the number instances from each class are roughly the same. The size and number of dimensions of
each dataset is shown in the first column of
Table~\ref{tab:label-noise}.

For each dataset, we randomly
consider $10$\% of the instances for test and perform
10-fold cross validation on the remaining part to find the optimal set of parameters for each method. These parameters include the $L_2$-regularizer values for all methods and temperature values for $t$-LR and $2$TLR. The regularizer values are selected from the range $[10^{-5}, 10^{-1}]$. The range of temperature values for $t$-LR is chosen to be $[1.12,1.9]$ and the range for $t_1$ and $t_2$ temperatures are set to $[0.1, 1]$ and $[1.12,1.9]$, respectively. The values of all parameters
are chosen using cross-validation. More specifically, the value of temperature for $t$-LR is set to $1.12$ for the CIFAR-10, Covertype, and Fashion MNIST, and $1.3$ for the Fonts dataset. For $2$TLR, we set $t_2$ to be the same as in $t$-LR and for $t_1$, we use $0.1$ for CIFAR-10, Covertype and Fashion MNIST, and $0.9$ for the Fonts dataset. The results are averaged over $10$ random train-test splits. We perform experiments in the presence of instance and label noise. All experiments are done on a 24 core cluster with 128 GB of RAM. We use a parallel implementation which utilizes all the cores in a machine.

We use the L-BFGS method for minimizing the losses. The initial weights are set to values sampled from a zero-mean Gaussian distribution with std $= 0.001$. In general, $t$-LR and our $2$TLR method are non-convex and converge to a local minimum. However, the results are consistent over multiple random initializations. This can be verified by the std of the accuracy results in Table~\ref{tab:label-noise}. Note that we observed the method to converge to bad local minima for std $> 0.01$.

\subsection{Instance Noise}

For the instance noise experiments, we consider the case
where a subset of the training instances, chosen uniformly
at random are replaced by instances from the remaining set
of classes (i.e. those classes other than the two selected
classes for the binary classification). This resembles the
case of a multiclass dataset where a subset of the
instances from each class are mislabled as instances of
other classes. Therefore these mislabled instances often
become extreme outliers for the class they are wrongly
labeled with.

Figure~\ref{fig:outlier} shows the results in
the presence of different amounts of this type of instance noise. 
The new $2$TLR method is significantly more robust to this noise 
than all the other methods and its performance is not considerably 
affected by up to $33$\% noise. The main reason for robustness of our method is the fact that by capping the surrogate loss, the total loss of the method is not affected much by the loss of each individual instance. 
This also validates our claim that tail-heaviness of the
distribution by itself (as used in $t$-LR) cannot handle
the outliers as well: In some case $t$-LR provides even worse results than LR
and all are beaten by $2$TLR.

\subsection{Label Noise}

We consider the label noise experiments where the labels of a subset of the training instances is flipped. Note that unlike the instance noise which alters the input distribution $p(\x)$, the labels noise targets the distribution of the labels $p(c\,\vert\, \x)$. Therefore, the label noise  is generally handled by first approximating the label inversion rates and then, correcting the data distribution by reweighting the loss of individual instances or considering a label-dependent surrogate loss~\cite{reweighting,labelnoise}. Nevertheless, the noise can be alleviated to some extent by the tail-heaviness of the modeling distribution~\cite{tlogistic}. In addition to tail-heaviness, we show that in some cases, tuning the level of non-convexity and bounding the loss function also improves the performance. 

We consider the ``random'' label noise where the label of a uniformly sampled subset of points is flipped. The subset of the noisy instances can also be selected  by an adversarial mechanism that targets the training instances based on a certain notion of ``importance''. We also consider ``small-margin'' and ``large-margin'' label noise in which we first train a LR classifier on the noise free data and calculate the margin $c \cdot (\w^\top \x)$ of each datapoint. Next, we  select the desired portion of the correctly classified datapoints that receptively have the smallest and largest margins. Therefore, these two noise mechanisms target different type of instances, i.e. those closer to the decision boundary and those that are far away. Table~\ref{tab:label-noise} shows the results under $10$\% noise. $2$TLR consistently has superior performance in all cases on all datasets. In some cases, the optimal value of the temperatures coincides with the values for LR ($t_1 = t_2 = 1$) and $t$-LR ($t_1 = 1,\, t_2 > 1$). However, in most cases, the optimal performance is achieved when $0 < t_1 < 1$.

\subsection{Runtime}
Table~\ref{tab:runtime} shows the runtime of the optimization step of the methods.
In general, the runtime of the $2$TLR is comparable to the other methods, and in some cases the convergence time is faster than the vanilla logistic regression (LR). However, in some cases (e.g. Covertype dataset), $2$TLR takes considerably longer time to converge.
In particular, 
the overhead from calculating the $G_t$ values is negligible;\footnote{This was validated empirically by comparing to the \texttt{fzero} function in MATLAB, but the results are omitted.}; the iterative algorithm takes around 20 iterations to converge to an accuracy of $10^{-10}$.
%

\section{Conclusions}

We developed a generalized loss function for logistic regression which provides two temperatures to tune the properties of the loss. 
The first temperature tunes the level of non-convexity and the boundedness of the loss while the second one controls the tail-heaviness of the probabilities.
Our experiments indicate that tuning the level of the non-convexity 
and boundedness is a crucial property for obtaining
robustness to both instance and label noise while the
computation time is comparable to logistic regression. 
%

\section*{Acknowledgement}
The authors would like to thank Nan Ding for his help with the iterative algorithms for calculating the normalization constants.
\bibliographystyle{plain}
\bibliography{refs} 

\clearpage
\appendix
\numberwithin{equation}{section}

\section{Verification of Iterative Algorithms for Computing $G_t$}
\label{sec:gt_converge}
In this section, we verify that the iterative algorithm for computing $G_t$ is going to converge in the binary case. The proof for the multiclass case follows immediately as a simple extension. We only need to verify that $\tilde{a}_{(k)}$ converges to the corresponding $\tilde{a}$ of $a$ such that the value of $G_t$ normalizes the sum.

First of all, given $a$, since $t>1$ and $Z(\tilde{a})>1$, it is clear that $0<\tilde{a}<a$. On the domain of $0<u<a$, it is easy to verify that $Z(u)^{1-t} a- u$ is a monotonically decreasing function and it crosses at 0 only at $\tilde{a}$. Therefore, when $\tilde{a}_{(k)}>\tilde{a}$, $\tilde{a}_{(k+1)}<\tilde{a}_{(k)}$; when $\tilde{a}_{(k)}<\tilde{a}$, $\tilde{a}_{(k+1)}>\tilde{a}_{(k)}$. 

We then prove that $\tilde{a}_{(k)}$ is a monotonically decreasing sequence. We prove this by mathematical induction. 
Since $\tilde{a}_{(0)}=\ah$, $\tilde{a}_{(1)}<a=\tilde{a}_{(0)}$. 
Next assume that in the $k$-th iteration, $\tilde{a}_{(k)}<\tilde{a}_{(k-1)}$. Since $Z(\tilde{a}_{(k)})>Z(\tilde{a}_{(k-1)})$, we have $\tilde{a}_{(k+1)}<\tilde{a}_{(k)}$. 
Therefore, it follows that $\tilde{a}_{(k)}$ is monotonically decreasing and it is lower bounded by $\tilde{a}$. Furthermore, $\lim_{k \to +\infty} \tilde{a}_{(k)}$ exists. 

Finally, 
\begin{align}
\lim_{k \to +\infty} \tilde{a}_{(k)} &= \lim_{k \to +\infty} \tilde{a}_{(k+1)} \nonumber\\
 &= \lim_{k \to +\infty} Z(\tilde{a}_{(k)})^{1-t} a \nonumber\\
 &= Z(\lim_{k \to +\infty} \tilde{a}_{(k)})^{1-t} a, \label{eq:gt_converge_verify}
\end{align}
where \eqref{eq:gt_converge_verify} holds because $Z(u)^{1-t}$ is continuous in $u$. Therefore, it follows that $\lim_{k \to +\infty} \tilde{a}_{(k)} = \tilde{a}$. 

For the binary case when $t = 2$, note that
$$
\exp_t(x) = (1 - x)^{-1} \,\,\text{ and }\,\, \log_t(x) = 1 - x^{-1}\, .
$$
The value $G_t(a)$ needs to satisfy
\begin{align*}
    1 & = \exp_t(\frac{a}{2} - G_t(a)) + \exp_t(-\frac{a}{2} - G_t(a))\\
    & = \frac{1}{1 + \sfrac{a}{2} + G_t(a)} + \frac{1}{1 - \sfrac{a}{2} + G_t(a)}\\
    & = \frac{2\, (1 + G_t(a))}{(1 + G_t(a))^2 - \sfrac{a^2}{4}}\, ,
    \intertext{which yields}
    & (1 + G_t(a))^2 - \frac{a^2}{4} = 2\,(1 + G_t(a))\, .
\end{align*}
 By cancelling the terms from both sides, we have
 $$
 G_t(a)^2 = \frac{a^2}{4} + 1\, .
 $$
 Since $G_t(a) \geq 0$, we have $ G_t(a) = \sqrt{\sfrac{a^2}{4} + 1}$.
\section{Proof of Remark~\ref{th:props}}
\label{app:props}

For the surrogate loss $$\sur_{t_1}^{t_2}(a) = -\log_{t_1}
\exp_{t_2}(a/2 - \lp(a)),$$ we have
\begin{align}
&\frac{\partial \sur_{t_1}^{t_2}(a)}{\partial a} =\, - \ph(a)^{t_2 - t_1}\, \left(\frac{1}{2} - \partial\lp(a)\right)\, ,\nonumber\\
&\frac{\partial^2 \sur_{t_1}^{t_2}(a)}{\partial a^2}  =\, \ph(a)^{t_2 - t_1}\, \times \label{eq:loss-second}\\
&\left[\partial^2 \lp(a)- (t_2 - t_1)\,
\ph(a)^{t_2- 1}\left(\frac{1}{2} - \lp(a)\right)^2\right]\, , \nonumber
\end{align}
where we define $\ph(a) \coldef \exp_{t_2}(a/2 - \lp(a))$ and $\partial\lp(a)$ and $\partial^2 \lp(a)$ are given as follows. 
\begin{align}
\partial \lp(a) & \!=\!  \frac{1}{2}\!\frac{\sum_c c\,
\exp_{t_2}(\frac{c}{2}a\!-\!\lp(a))^{t_2}}{\sum_c
\exp_{t_2}(\frac{c}{2}a\!-\!\lp(a))^{t_2}}\, ,
\label{eq:log-part-first}
\end{align}
\begin{equation}
\partial^2 \lp(a) =
 \frac{t_2\! \sum_c \!\exp_{t_2}(\frac{c}{2} a \!-\! \lp(a))^{2t_2 -1}
\left[\frac{c}{2}\!-\! \partial \lp(a)\right]^2 }{\sum_c
\exp_{t_2}(\frac{c}{2} a \!-\! \lp(a))^{t_2}}\label{eq:log-part-sec}
.\!\!\!\!
\end{equation}

For $t_2 = t_1 \geq 1$, we have
\begin{equation*}
\frac{\partial^2 \sur_{t_1}^{t_2}(a)}{\partial a^2}  = \partial^2 \lp(a) \geq 0\, ,
\end{equation*}
which can be verified from~\eqref{eq:log-part-sec}. Moreover, for $t_1 \geq 1$ and $t_1 \geq t_2$, we have
\begin{align}
\frac{\partial^2 \sur_{t_1}^{t_2}(a)}{\partial a^2} &= \frac{1}{\ph(a)^{t_1 - t_2}}\, \times\nonumber\\
& \left[\partial^2 \lp(a)+ (t_1 - t_2)\, \ph(a)^{t_2- 1}\left(\frac{1}{2} - \lp(a)\right)^2\right]\nonumber\\
 &\geq \partial^2 \lp(a)+ (t_1 - t_2)\, \ph(a)^{t_2- 1}\left(\frac{1}{2} - \lp(a)\right)^2\nonumber\\
 & \geq \partial^2 \lp(a) \geq 0\, . \label{eq:more-convex}
\end{align}
Thus, the loss is convex, similar to the latter case.

Now, consider the case $t_2 \geq t_1$. Suppose $\ph(-a) = (1 - \ph(a)) = \lambda\, \ph(a)$ for some $\lambda \geq 0$. Substituting for $\ph(-a)$ in~\eqref{eq:log-part-first} and~\eqref{eq:log-part-sec}, we can write~\eqref{eq:loss-second} as
\begin{multline}
\frac{\partial^2 \sur_{t_1}^{t_2}(a)}{\partial a^2} = \ph(a)^{t_2 - 1}\, \frac{1}{(1 + \lambda^{t_2})^2}\,\\
\times \left[t_2 \left(\frac{1 + \frac{1}{\lambda}}{1 +
\lambda^{t_2}}\right) - (t_2 - t_1)\right]\, .\nonumber
\end{multline}
For sufficiently small (respectively, large) value of $\lambda$, we have $\frac{\partial^2 \sur_{t_1}^{t_2}(a)}{\partial a^2} > 0$ (respectively, $\frac{\partial^2 \sur_{t_1}^{t_2}(a)}{\partial a^2} < 0$). The inflection point happens when $t_2 (1 + \frac{1}{\lambda}) = (t_2 - t_1)(1 + \lambda^{t_2})$, i.e.~$\frac{\partial^2 \sur_{t_1}^{t_2}(a)}{\partial a^2} = 0$.

Finally, we show the case $t_1 < 1$. We only need to consider the case $t_2 \leq t_1< 1$. Note that for the binary case,
\begin{equation}
\label{eq:sum-prob-one}
\exp_{t_2}(a/2 - \lp(a)) + \exp_{t_2}(-a/2 - \lp(a)) = 1\, .
\end{equation}
Using the definition of $\exp_{t_2}$, we can write~\eqref{eq:sum-prob-one} as
\begin{multline}
\label{eq:sum-exp-one}
[1 + (1-t_2)\,(a/2 - \lp(a))]_+^{1/(1-{t_2})}\\
 + [1 + (1-t_2)\,(-a/2 - \lp(a))]_+^{1/(1-{t_2})} = 1\, .
\end{multline}
For $a = 0$,~\eqref{eq:sum-exp-one} yields
\begin{equation*}
[1 + (1-t_2)\,(-\lp(0))]_+^{1/(1-{t_2})} = \frac{1}{2}\, .
\end{equation*}
From $t_2 < 1$, we have $(1-t_2) > 0$ and therefore, $\lp(0) > 0$. From convexity and symmetry ($\lp(a) = \lp(-a)$) conditions, we conclude $\lp(a) \geq \lp(0) \geq 0, \, \forall a$. Consequently, for values of $a \leq -\frac{1}{(1-t_2)}$, $\lp(a) = -\frac{a}{2}$ satisfies~\eqref{eq:sum-prob-one}. This implies that for  $a \leq -\frac{1}{(1-t_2)}$, we have $\ph(a) = 0$ and thus, $\sur_{t_1}^{t_2}(a) = -\log_{t_1}(0) = -\frac{1}{1-t_1}$ is a constant. From~\eqref{eq:more-convex}, we conclude that the loss is convex for $a > -\frac{1}{(1-t_2)}$ and is a constant for $a \leq -\frac{1}{(1-t_2)}$ Thus, it is quasi-convex.

\end{document}